\documentclass[11pt,twoside]{article}
\usepackage[utf8]{inputenc}
\usepackage[hmargin=1in,vmargin=1in]{geometry}

\usepackage{microtype}
\usepackage{graphicx}
\usepackage{subfigure}
\usepackage{amsfonts}
\usepackage{booktabs} % for professional tables

\usepackage{amsmath}
\usepackage{amsthm}

\newtheorem{lemma}{Lemma}[section]
\newtheorem{theorem}[lemma]{Theorem}

\newcommand{\opt}{\textsc{OPT}}
\newcommand{\eps}{\varepsilon}
\newcommand{\nc}{\overline{c}}
\newcommand{\ignore}[1]{}
\newcommand{\etal}{\textit{et al}.}

\title{A Push-Relabel Based Additive Approximation for Optimal Transport}

\author{%
  Nathaniel Lahn%
  \thanks{{Radford University}, {nlahn@radford.edu}.},\qquad
  Sharath Raghvendra%
  \thanks{{Virginia Tech},
  {sharathr@vt.edu}.
  },\qquad
  Kaiyi Zhang%
  \thanks{{Virginia Tech}, {kaiyiz@vt.edu}.}.\thanks{Following convention in theoretical computer science, authors are listed alphabetically by last name.}
}

\date{}

\begin{document}

\maketitle

\begin{abstract}
Optimal Transport is a popular distance metric for measuring similarity between distributions. Exact algorithms for computing Optimal Transport can be slow, which has motivated the development of approximate numerical solvers (e.g. Sinkhorn method). We introduce a new and very simple combinatorial approach to find an $\varepsilon$-approximation of the OT distance. Our algorithm achieves a near-optimal execution time of $O(n^2/\varepsilon^2)$ for computing OT distance and, for the special case of the assignment problem, the execution time improves to $O(n^2/\varepsilon)$. Our algorithm is based on the push-relabel framework for min-cost flow problems.

Unlike the other combinatorial approach (Lahn, Mulchandani and Raghvendra, NeurIPS 2019) which does not have a fast parallel implementation, our algorithm has a parallel execution time of $O(\log n/\varepsilon^2)$. 
Interestingly, unlike the Sinkhorn algorithm, our method also readily provides a compact transport plan as well as a solution to an approximate version of the dual formulation of the OT problem, both of which have numerous applications in Machine Learning.
For the assignment problem, we provide both a CPU implementation as well as an implementation that exploits GPU parallelism. Experiments suggest that our algorithm is faster than the Sinkhorn algorithm, both in terms of CPU and GPU implementations, especially while computing matchings with a high accuracy. 
\end{abstract}

\section{Introduction}
Optimal transport (OT) is a useful metric for measuring similarity between distributions and has numerous applications~\cite{app1,app2,app3,app4,app5,app6}. Given two distributions $\mu$ and $\nu$, this metric captures the minimum-cost plan for transporting mass from $\mu$ to $\nu$. For discrete distributions, the optimal transport problem can be formulated and solved as a linear program. Exact solutions often use a primal-dual approach, i.e., they find a feasible solution to the LP while maintaining a feasible solution for the dual formulation of the LP. When $\mu$ and $\nu$ are not known (but we have access to samples from them) or when they are continuous, one can take $n$ samples from $\mu$ and $\nu$, assign each sample a weight of $1/n$, and compute the optimal transport distance between the samples. This special case, where the weight of every sample point is equal, is called the \emph{assignment} problem. The so-called Hungarian method can be used to solve this special case in $\Theta(n^3)$ time~\cite{hungarian}. 

Various aspects of OT have found applications in machine learning.  
The OT cost can be used to measure similarity between images and for image retrieval tasks. 
 The transport plan itself can be used to interpolate between distributions~\cite{color}. Furthermore, a solution to the assignment problem as well as its dual formulation have been used in GAN training~\cite{gans}. However, exact solvers take too much time for practical purposes.

This has motivated the study of approximation algorithms including numerical and combinatorial methods that generate an $\eps$-approximate transport plan in $\Omega(n^2/\eps)$ time. Perhaps the most popular among these is an entropy regularized version of the optimal transport that can be solved using the Sinkhorn method~\cite{altschulerNIPS17, sinkhorn}. The simplicity of the approach has resulted in its wide use. This algorithm, for instance, benefits from high parallelism making it really scalable, especially in the presence of resources such as GPUs. On the flip side, the entropy regularization requires the use of an exponential function. This may lead to numerical instabilities when $\eps$ becomes small.

One can adapt the combinatorial exact algorithm to quickly find an approximation of the optimal transport. For instance, \cite{our-neurips-2019-otapprox} provided a non-trivial modification of the classical Gabow-Tarjan algorithm for the transportation problem. Their algorithm ran no more than $\lfloor 2/\eps \rfloor + 1$ iterations, where each iteration executed a Dijkstra's shortest path search to find and augment along a set of ``augmenting paths". In the sequential setting, it outperforms the numerical methods, especially when $\eps$ is small.  Unfortunately, similar to the GT-algorithm, this method may need $\Omega(n)$ parallel time. This is because each of the $\Omega(n)$ flow augmentations have to be done in a sequential manner. Furthermore, shortest path algorithms, which are required for computing augmenting paths with small costs, are also hard to parallelize. 

In this paper, we design a new and efficient push-relabel based optimal transport algorithm that does not have issues such as the numerical instabilities displayed by the Sinkhorn method. It also has a parallel execution time of $O(\log n/\eps^2)$, unlike the graph-theoretic algorithm of~\cite{our-neurips-2019-otapprox}, thus resolving an open question posed by Lahn~\etal~\cite{our-neurips-2019-otapprox}. Next, we more formally introduce the optimal transport problem for discrete distributions.

In the optimal transport problem, we are given two discrete distributions $\mu$ and $\nu$ whose supports are the point sets $A$ and $B$, respectively. For each point $a \in A$ (resp. $b \in B$), we associate a probability of $\mu_a$ (resp. $\nu_b$) with it such that $\sum_{a\in A}\mu_a = \sum_{b\in B}\nu_b = 1$. We refer to each point of $A$ as a demand point and each point in $B$ as a supply point. For any edge $(a,b) \in A\times B$, we are given a cost $c(a,b)$; we assume that the costs are scaled so that the largest cost edge is $1$. Let $\beta c(a,b)$ be the cost of transporting a supply amount of $\beta$ 
from $b$ to $a$.  A transport plan is a function $\sigma:A\times B\rightarrow R_{\ge0}$ that assigns a non-negative value to each edge of $G$, indicating the amount of supply transported along the edge.  The transport plan $\sigma$ is such that the total supplies transported into (resp. from) any demand (resp. supply) node $a \in A$ (resp. $b\in B$) is bounded by the demand (resp. supply) at $a$ (resp. $b$). 
The cost of the transport plan, denoted by $w(\sigma)$, is  given by $\sum_{(a,b) \in A\times B}\sigma(a,b)c(a,b)$. In this optimal transport problem, we are interested in finding a minimum-cost transport plan that transports all of the supply, denoted by $\sigma^*$. 
We also define an $\eps$-approximate transport plan to be any transport plan $\sigma$ with a cost $w(\sigma) \le w(\sigma^*)+\eps$ that transports all of the supply. 

The special case where $A$ and $B$ each contain $n$ points and where every point in $A$ (resp. $B$) has a demand of $1/n$ (resp. supply of $1/n$) is called the \emph{assignment problem}. In this special case, there is an optimal transport plan with a special structure; specifically, one where all edges $(a,b)$ with $\sigma(a,b)> 0$ are vertex-disjoint, i.e., they form a matching. 
 We say that a matching on $G$ is \emph{perfect} if it has $n$ edges.

For simplicity in exposition, we will assume that, for the assignment problem, all demands and supplies are $1$.
The cost of any matching $M$, denoted by $c(M)$ is the total cost  of all of its edges, i.e.,
$$c(M) = \sum_{(a,b) \in M} c(a,b).$$ 
In the assignment problem, we are interested in finding a minimum-cost perfect matching $M^*$. For an $\eps >0$, we say that a matching $M$ is an additive $\eps$-approximation if $c(M) \le c(M^*)+ \eps n$. 
 Therefore, the error allowable in an $\eps$-approximate matching is $+\eps n$.

\paragraph{Our Results:} In this paper, we obtain the following results:
\begin{itemize}
    \item We present a very simple algorithm to compute an $\eps$-approximate transport plan in $O(n^2/\eps^2)$ time. 
For the special case of assignment problem, our algorithm's execution time improves to $O(n^2/\eps)$. Our algorithm is based on the popular push-relabel framework for computing minimum-cost flow.
    \item Our algorithm maintains a matching $M$ and a set of dual weights $y(\cdot)$ on vertices of $A\cup B$. The algorithm runs in $O(1/\eps^2)$ iterations and in each iteration, it executes three steps. First, it greedily computes a maximal matching $M'$. In the second step, it uses $M'$ to update the matching $M$ (the push step). Finally, it updates the dual weights (relabel step). The push and relabel step take only $O(n)$ sequential time and $O(1)$ parallel time. The only non-trivial step, therefore, is the computation of a greedy maximal matching which can be done in $O(n^2)$ sequential time and $O(\log n)$ parallel time.
      
    \item For the assignment problem, we provide a CPU implementation as well as an implementation that exploits GPU parallelism. Experiments suggest that both the CPU and GPU implementations of our algorithm outperform corresponding CPU and GPU implementations of the Sinkhorn algorithm provided by the Python Optimal Transport library~\cite{POT} in terms of running time, while achieving the same level of accuracy.
\end{itemize} 

Computing maximal matching has been studied in several different models, such as the massively parallel computation model, by both theoreticians~\cite{behnezhad2019exponentially} and practitioners~\cite{auer2012gpu}. As a result, it is likely that a more careful parallel implementation may
 result in even further improvements in terms of running time. Our current CPU and GPU implementations are only for the assignment problem. We are in the process of extending our implementation to the optimal transport problem and will include the results in the next version of this paper.

\paragraph{Organization:} In Section~\ref{sec:prelim}, we present the definitions required to describe our algorithm. In Section~\ref{sec:main-algo}, we present our algorithm for the assignment problem. We present an algorithm that computes a $3\eps$-approximation of the optimal transport. To obtain an $\eps$-approximation, one can simply choose the error factor in the algorithm to be $\eps/3$. In Section~\ref{sec:analysis}, we prove the sequential and parallel complexity of our algorithm for the assignment problem. In Section~\ref{sec:transport}, we extend our algorithm to the optimal transport problem. Finally, we present the experimental results in Section~\ref{sec:experiments}.

\section{Algorithm}
In this section, given an input to the assignment problem and a value $0 < \eps < 1$, we present an algorithm that computes a $3\eps$-approximate matching.
\subsection{Preliminaries}
\label{sec:prelim}
We begin by introducing the terminologies required to understand our algorithm for the assignment problem. For any matching $M$, we say that any vertex $v \in A\cup B$ is \emph{free} if $v$ is not matched in $M$ and \emph{matched} otherwise. Our algorithm critically uses the notion of a maximal matching which we introduce next. For any bipartite graph that is not necessarily complete, any matching $M$ is \emph{maximal} if and only if at least one end point of every edge in the graph is matched in $M$. Thus, if a matching is not maximal, there is at least one edge between two free vertices. One can, therefore, compute a maximal matching in a greedy fashion by iteratively picking such an edge and adding it to the matching.

For every edge $(u,v) \in A\times B$, we transform its cost so that it becomes an integer multiple of $\eps$ as follows:
\begin{equation}
    \nc(u,v) = \eps \lfloor c(u,v)/\eps \rfloor
\end{equation}
The rounding of edge costs may introduce an error that is bounded by $\eps$ for each edge and by at most $\eps n$ for any matching. Our algorithm assigns a dual weight $y(v)$ for every $v \in A\cup B$ such that a set of relaxed dual feasibility conditions are satisfied. A matching $M$ along with dual weights $y(\cdot)$ is $\eps$-feasible if, for every edge $(a,b) \in A \times B$, 
\begin{eqnarray}
    &y(a) + y(b) \leq \nc(a,b) +\eps\quad &\mathrm{if\ } (a,b) \notin M \label{eq:feas1}\\
    &y(a) + y(b) = \nc(a,b) \quad &\mathrm{if \ } (a,b) \in M \label{eq:feas2}
\end{eqnarray}
In Lemma~\ref{lem:error-balanced}, we show that any $\eps$-feasible matching produced by our algorithm has a cost within an additive error of $\eps$ from the optimal solution with respect to the costs $\nc(\cdot, \cdot)$. For any edge $(u,v)$, we define its \emph{slack} $s(u,v)$ to be $0$ if $(u,v) \in M$. Otherwise, if $(u,v) \not\in M$, we set its slack to be $s(u,v)=\nc(u,v) - y(u) - y(v)$. We say that $(u,v)$ is admissible if the slack on the edge is $0$. 

We observe that any matching $M$ whose cardinality is at least $(1-\eps)n$ can be converted into a perfect matching simply by arbitrarily matching the remaining $\eps n$ free vertices. The cost of any edge is at most $1$, and so, this increases the cost of the matching $M$ by at most $\eps n$. In addition to this, the rounding of costs from $c(\cdot,\cdot)$ to $\nc(\cdot,\cdot)$ also introduces an increase of cost by $\eps n$. Finally, the $\eps$-feasibility conditions introduced an additional additive error of $\eps n$, for a total error of $3\eps n$, as desired. Thus, in the rest of this section, we present an algorithm that computes an $\eps$-feasible matching of cardinality at least $(1-\eps)n$, which will has a cost no more than $\eps n$ above the optimal matching's cost with respect to $\nc(\cdot, \cdot)$.

\subsection{Algorithm}
\label{sec:main-algo}
Initially, we set the dual weight of every vertex $b \in B$ to be $\eps$ and every vertex $a\in A$ to be $0$. We initialize $M$ to $\emptyset$. Our initial choice of $M$ and the dual weights satisfies~\eqref{eq:feas1} and~\eqref{eq:feas2}. Our algorithm executes iterations, which we will call \emph{phases}. Within each phase, the algorithm constructs the set $B'$, which consists of all free vertices of $B$. If $|B'| \le \eps n$, then $M$ is an $\eps$-feasible matching of cardinality at least $(1 - \eps)n$, and the algorithm will arbitrarily match the remaining free vertices and return the resulting matching. Otherwise, the algorithm computes the subset $E' \subseteq E$ of admissible edges with at least one end point in $B'$. Let $A' = \{a \mid  a \in A \textrm{ and } (a,b) \in E'\}$, i.e., the set of points of $A$ that participate in at least one edge in $E'$. For each phase, the algorithm executes the following steps:
\begin{itemize}
    \item [(I)] {\it Greedy step:} Computes a maximal (i.e., greedy) matching $M'$ in the graph $G'(A'\cup B', E')$.
    \item[(II)] {\it Matching Update:} Let $A''$ be the set of points of $A'$ that are matched in both $M$ and $M'$ and let $M''$ be the edges of $M$ that are incident on some vertex of $A''$. The algorithm adds the edges of $M'$ to $M$ and deletes the edges of $M''$ from $M$.
    \item[(III)] {\it Dual Update:} 
    \begin{itemize} 
    \item[a.] For every edge $(a,b) \in M'$, the algorithm sets $y(a) \leftarrow y(a)-\eps$, and
    \item[b.] For every vertex $b \in B'$ that is free with respect to $M'$, the algorithm sets $y(b) \leftarrow y(b) +\eps$.
    \end{itemize}
\end{itemize}

In each phase, the matching update step will add edges of $M'$ to $M$ and remove edges of $M''$ from $M$. By construction, the updated set $M$ is a matching. Furthermore, every vertex of $A$ that was matched prior to the update continues to be matched after the update.

\begin{lemma}
\label{lem:match-update}
The new set $M$ of edges obtained after Step (II) is a matching. Furthermore, any vertex of $A$ that was matched prior to Step (II) will continue to be matched after the execution of Step (II).
\end{lemma}

The dual update step increases or reduces dual weights by $\eps$. Therefore, the dual weights always remain an integer multiple of $\eps$. 

The algorithm maintains the following invariants:

\begin{itemize}
    \item [(I1)] The dual weight of every vertex in $B$ (resp. $A$) is non-negative (resp. non-positive). Furthermore, every free vertex of $A$ has a dual weight of $0$.
    \item[(I2)] The matching $M$ and a set of dual weights $y(\cdot)$ is $\eps$-feasible.
\end{itemize}

Note that (I1) and (I2) are true at the start of the algorithm. Inductively assume that the invariants hold at the start of any phase. We show that the invariants continue to hold at the end of the phase.

\paragraph{Proof of (I1):} 
The dual update step (III) only increases the dual weight of $B$ and reduces the dual weight of $A$. Therefore, the dual weights of $A$ remain non-positive and the dual weights of $B$ remain non-negative. Next, we show that every free vertex of $A$ with respect to the matching $M$ continues to have a dual weight of $0$ at the end of this phase. Clearly, step (I) does not affect this property. Step (II) updates the matching $M$. From Lemma~\ref{lem:match-update}, every matched vertex of $A$ remains matched after the execution of Step (II). So, every free vertex of $A$ continues to have a dual weight of $0$. During step (III (a)),
the dual weight for any vertex $a \in A$ reduces only if it is matched in $M'$. By construction, step (II) will add the edges of $M'$ to $M$. So, every vertex of $A$ whose dual weight is updated by (III) is matched in $M$ after updating $M$.  
Consequently, the dual weights of all free vertices of $A$ remain unchanged, i.e., their dual weight remains $0$ implying (I1).

\paragraph{Proof of (I2):} 
To show that the edge remains feasible, we need to show that the slack on every edge remains non-negative and all edges of $M$ have zero slack. During each phase, the feasibility conditions could potentially be affected for two reasons. First, the matching $M$ gets updated in step (II). Second, the dual weights of certain points change in Step (III). At the start of the phase, consider an edge $(a,b)$ that is feasible but not admissible, i.e., it has a slack of at least $\eps$. The matching $M'$ consists only of admissible edges, and, therefore, $(a,b) \not\in M$. Observe that the dual weight of $b$ may increase by $\eps$ in Step (III(b)) reducing its slack by $\eps$. The slack, however, remains non-negative. In Step (III(a)), the dual weight of $a$ may reduce by $\eps$, only increasing the slack on $(a,b)$. Therefore, the edge $(a,b)$ continues to be feasible and satisfies~\eqref{eq:feas2}.

Next, consider an edge $(a,b)$ that is admissible at the start of the phase. At the end of the phase, there are two possibilities: either $(a,b)$ is in the matching, or it is not in the matching.  
We begin by showing that if $(a,b)$ is in the updated matching $M$ after Step (II), then $(a,b)$ satisfies~\eqref{eq:feas2} after Step (III).

There are two possibilities: (i) Prior to the matching update step, the edge $(a,b)$ was in $M$, or (ii) $(a,b)$ is an edge in $M'$ and is added to $M$ by Step (II).

In case (i), the dual weights of $a$ and $b$ remain unchanged in Step (III) and therefore, $(a,b)$ satisfies~\eqref{eq:feas2}. 
For case (ii), observe that, at the start of the phase, $(a,b) \in E'$ is a non-matching admissible edge. Therefore, the edge $(a,b)$ satisfies 
\begin{equation*}
y(a) + y(b) = \nc(u,v) +\eps. \label{eq:admissible}
\end{equation*} 
The update of (III(b)) does not change the dual weight of $b$, and the update step (III(a)) reduces the dual weight of $a$: $y(a) \leftarrow y(a) - \eps$. After step (III), therefore, we have $(a,b) \in M$ satisfying~\eqref{eq:feas2}.  Thus, at the end of the phase, every edge of $M$ satisfies~\eqref{eq:feas2}.

Next, we show that, for every edge $(a,b)$ that is both admissible at the start of the iteration and is a non-matching edge at the end of the iteration, $(a,b)$ satisfies~\eqref{eq:feas1}. There are two possibilities: (i) $b \not\in B'$, and, (ii)  $ b \in B'$.

In case (i), at the start of the phase the edge satisfies~\eqref{eq:feas1}; note that matching edges at the start of the phase satisfy~\eqref{eq:feas2}, which means~\eqref{eq:feas1} is also satisfied. During the dual update step, only points of $B'$ may undergo a dual update. Since $b \not\in B'$, $y(b)$ remains unchanged. The dual weight of $a$ may reduce by $\eps$, but this only increases the slack of $(a,b)$. Therefore~\eqref{eq:feas1} continues to hold. 

For case (ii), since $(a,b)$ is admissible and remains a non-matching edge at the end of the phase, $(a,b)\not\in M'$. Since $M'$ is a maximal matching, $a$ is matched to another vertex $b'$ in $M'$, i.e., $(a,b')$ is in $M'$. While the dual weight of $b$ may increase by $\eps$ in Step (III(b)), step (III(a)) will reduce the dual weight of $a$ (since $a$ is matched in $M'$) by $\eps$, ensuring that the slack on $(a,b)$ remains at least $0$. This completes the proof of (I2).

\section{Analysis}
\label{sec:analysis}
Next, in Section~\ref{sec:accuracy}, we use invariants (I1) and (I2) to show that the algorithm produces a matching with the desired accuracy. In Section~\ref{sec:efficiency}, we use the invariants to bound the sequential and parallel execution times of our algorithm.

\subsection{Accuracy}
\label{sec:accuracy}
As stated in Section~\ref{sec:prelim}, the rounding of costs from $c(\cdot,\cdot)$ to $\nc(\cdot,\cdot)$ introduces an error of $\eps n$. Furthermore, after obtaining a matching of size at least $(1 - \eps)n$, the cost of arbitrarily matching the last $\eps n$ vertices is no more than $\eps n$. From the following lemma, we can conclude that the total error in the matching computed by our algorithm is no more than $+3\eps n$. 

\begin{lemma}
\label{lem:error-balanced}
The $\eps$-feasible matching of size at least $(1 - \eps) n$ that is produced by the main routine of our algorithm is within an additive error of $\eps n$ from the optimal matching with respect to the rounded costs $\nc(\cdot,\cdot)$
\end{lemma}
\begin{proof}
Let $M$ be the matching produced by our algorithm, and let $M_{\opt}$ be an optimal matching with respect to the cost function $\nc(\cdot, \cdot)$. From~\eqref{eq:feas1}, and the fact that the dual weights of all free vertices with respect to $M$ are non-negative we have $\sum_{(a,b) \in M}\nc(a,b) = \sum_{(a,b) \in M}y(a) + y(b) \leq \sum_{v \in A \cup B}y(v)$. Note that $M_{\opt}$ is a perfect matching, and so, from~\eqref{eq:feas2}, we get $\sum_{v \in A \cup B}y(v) = \sum_{(a,b) \in M_{\opt}} y(a) + y(b) \leq \sum_{(a,b) \in M_{\opt}} \nc(a,b) + \eps n$. Combining these two observations completes the proof of the lemma.
\end{proof}

\subsection{Efficiency}
\label{sec:efficiency}
Suppose there $t$ phases executed by the algorithm. We use $n_i$ to denote the size of $B'$ in phase $i$. By the termination condition, each phase is executed only if $B'$ has more than $\eps n$ vertices, i.e.,  $n_i > \eps n$. First, in Lemma~\ref{lem:magbound}, we show that the magnitude of the dual weight of any vertex cannot exceed $(1+2\eps)$. This means the total dual weight magnitude over all vertices is upper bounded by $n(1+2\eps)$. Furthermore, in Lemma~\ref{lem:perphase} we show that, during phase $i$, the total dual weight magnitude increases by at least $\eps n_i$. From this, we can conclude that 
\begin{equation}
    \sum_{i=1}^t n_i \le n(1+2\eps)/\eps = O(n / \eps).\label{eq:total-free-vertices}
\end{equation}

Note that, since each $n_i \ge \eps n$, we immediately get $t\eps n \le n(1+2\eps)/\eps$, or $t \le (1+2\eps)/\eps^2 = O(1 / \eps^2)$. In order to get the total sequential execution time, we show, in Lemma~\ref{lem:efficientphase} that each phase can be efficiently executed in $O(n\times n_i)$ time. Combining this with equation~\eqref{eq:total-free-vertices} gives an overall sequential execution time of $O(n (\sum_{i=1}^t n_i)) = O(n^2 / \eps)$.

\begin{lemma}
\label{lem:magbound}
For any vertex $v \in A\cup B$, the magnitude of its dual weight cannot exceed $1+2\eps$, i.e., $|y(v)| \le (1+2\eps)$.
\end{lemma}
\begin{proof}
Since the algorithm only increases dual weight magnitudes, it is sufficient to show that the claim holds at the end of the algorithm. First, we show that the claim holds for all vertices of $B$. Let $a$ be an arbitrary free vertex of $A$ at the beginning of the last phase of the algorithm. From invariant (I1), the dual weight $y(a)$ is $0$. For every vertex $b \in B$, the edge $(a,b)$ satisfies either equation~\eqref{eq:feas1} or equation~\eqref{eq:feas2}, and, therefore, $y(b) \leq \bar{c}(a,b) + \eps - y(a) = \bar{c}(a,b) + \eps \leq 1 + \eps$. Now, observe that the maximum dual weight magnitude increase for any vertex during a single iteration, including the final iteration, is $\eps$. Thus, the maximum possible value of $y(b)$ for any $b \in B$ at the end of the algorithm is $1 + 2\eps$. 

Next, we show that, for any vertex $a \in A$, $|y(a)| < 1 + 2\eps$. If $a$ is free, this holds true from invariant (I1). Otherwise, $a$ is matched to some $b \in B$, and the edge $(a,b)$ is $\eps$-feasible, implying that $y(a) = \bar{c}(a,b) - y(b) \geq -y(b) \geq -(1 + 2\eps)$. Thus, $|y(a)| \leq 1 + 2\eps$ for every $a \in A$.
\end{proof}

\begin{lemma}
\label{lem:perphase}
The sum of the magnitude of the dual weights increases by at least $\eps n_i$ in each iteration.
\end{lemma}
\begin{proof}
Let $b$ be a vertex of $B'$ at the beginning of some iteration. If $b$ remains free at the end of the iteration, then $|y(b)|$ increases by $\eps$ during that iteration. Otherwise, $b$ was matched to a vertex $a$ in $M'$ during this iteration. This implies that $|y(a)|$ increases by $\eps$ during that iteration. Therefore, each vertex of $B'$ causes some vertex to experience an increase to its dual weight magnitude, and the total dual weight increase is at least $\eps n_i$.
\end{proof}

\begin{lemma}
\label{lem:efficientphase}
The execution time of each phase is $O(n\times n_i)$ time.
\end{lemma}
\begin{proof}
First, we note that the set of free vertices can easily be found in $O(n)$ time at the beginning of each phase. It is easy to see that steps (II) and (III) of the algorithm can be implemented to run in $O(n)$ time, since any matching has size $O(n)$ and each vertex experiences at most one dual adjustment per iteration. It remains to describe how step (I) can be implemented to run in $O(n \times n_i)$ time. In this step, the algorithm computes a maximal matching on the graph $G'(A' \cup B', E')$, where $E'$ is the set of admissible edges with at least one point in $B'$. Note that this graph can be constructed in $O(n \times n_i)$ time by scanning all edges incident on $B'$. Next, the algorithm finds a maximal matching $M'$ in $G'$. This matching $M'$ is found by processing each free vertex $b$ of $B'$ in an arbitrary order. The algorithm attempts to match $b$ by identifying the first edge $(a,b)$ in $G'$ such that $a$ is not already matched in $M'$. If such an edge $(a,b)$ is found, then $(a,b)$ is added to $M'$, and the algorithm processes the next vertex of $B'$. Otherwise, there is no way to add an edge incident on $b$ to $M'$. After all vertices of $B'$ are processed, $M'$ is maximal. Overall, the time for each phase is dominated by the time taken to compute a maximal matching, which is $O(n \times n_i)$.
\end{proof}

\paragraph{Parallel Efficiency}
Steps (II) and (III) of each phase of our algorithm are trivially parallelizable in $O(1)$ time. One can use the classical $O(\log n)$ time parallel algorithm for maximal matchings~\cite{israeli1986fast} to achieve a parallel execution time of $O(\log n)$ per phase and $O(\log n/\eps^2)$ in total. Maximal matchings can also be computed in massively parallel computation models; see for instance~\cite{behnezhad2019exponentially}.

\subsection{Analysis for the Unbalanced Case} In this section, we describe how the analysis of our matching algorithm can be extended to work for the unbalanced case, where $|A| \neq |B|$. This analysis is critical for proving the correctness of our optimal transport version of the algorithm. Without loss of generality, assume $|B| \leq |A| = n$. The overall description of the algorithm remains the same, except for the main routine of our algorithm produces an $\eps$-feasible matching of size at least $(1 - \eps)|B|$. The asymptotic running time of both the parallel and sequential algorithms remains unchanged. In the following lemma, we bound the additive error of our algorithm for the unbalanced case; the argument is very similar to Lemma~\ref{lem:error-balanced}.

\begin{lemma}
\label{lem:error-unbalanced}
Given an unbalanced input to the assignment problem with $|B| \leq |A|$, the $\eps$-feasible matching of cardinality at least $(1 - \eps)|B|$ that is returned by our algorithm is within an additive error of $\eps |B|$ from the optimal matching with respect to the cost function $\nc(\cdot,\cdot)$
\end{lemma}
\begin{proof}
Let $M$ be the matching produced by our algorithm, and let $M_{\opt}$ be an optimal matching with respect to the cost function $\nc(\cdot, \cdot)$. From~\eqref{eq:feas1}, and the fact that the dual weights of all free vertices with respect to $M$ are non-negative we have $\sum_{(a,b) \in M}\nc(a,b) = \sum_{(a,b) \in M}y(a) + y(b) \leq \sum_{v \in A \cup B}y(v)$. Note that $M_{\opt}$ is a maximum-cardinality matching, and so all vertices of $B$ are matched in $M_{\opt}$. Therefore, all vertices unmatched by $M_\opt$ are of type $A$, and the algorithm maintains that such vertices have a non-positive dual weight. Combining this with~\eqref{eq:feas2}, we get $\sum_{v \in A \cup B}y(v) \leq \sum_{(a,b) \in M_{\opt}} y(a) + y(b) \leq \sum_{(a,b) \in M_{\opt}} \nc(a,b) + \eps|B|$. Combining these two observations completes the proof of the lemma.
\end{proof}

\section{Extending the algorithm to OT}
\label{sec:transport}

Given an instance $\mathcal{I}$ of the optimal transport problem, one can scale the supplies and demand at each node by a multiplicative factor of $\theta$. Such a scaling will not affect the optimal transport plan but will increase its cost by a factor of $\theta$. Thus, to find an $\eps$-approximation, it suffices if we find a solution that is an additive error of $\eps\theta$ from the optimal.

Next, as in~\cite{our-neurips-2019-otapprox}, we choose $\theta = 4n/\eps$, round the supplies down to the closest integer, and round the demands up to the closest integer. Then, we create a matching instance by simply replacing each node with a demand $d_a$ (resp. supply $s_b$) with $d_a$ (resp. $s_b$) copies, each having a demand (resp. supply) of $1$. Let $\mathbb{A}$ and $\mathbb{B}$ be the multi-set of demand and supply nodes created by these copies. We then solve this new instance  $\mathcal{I}'$ of the \emph{unbalanced matching problem} using our algorithm from Section~\ref{sec:main-algo}. By our choice of $\theta$, the size of this new instance (in terms of number of vertices) is $\Theta((\theta+n)) = \Theta(n/\eps)$.  The time taken by our algorithm will be $O(n^2/\eps^3)$ sequential time and $O(\log (n/\eps)/\eps^2)$ parallel time. 
To speed the sequential algorithm by a factor of $1/\eps$, we explicitly maintain the following invariant: We raise the dual weight of any unmatched free supply node $b \in \mathbb{B}$ to be at least as large as the largest dual weight among all its other copies $b' \in \mathbb{B}$. This can be enforced in a straight-forward way while updating the matching. Furthermore, this change in dual weight does not violate any feasibility conditions.   

Given this property, like in~\cite{lr_jocg,sa_soda12}, we show in Lemma~\ref{lem:twoclusters} that copies of the same vertex can have no more than two distinct dual weights at any time. Thus, we can maintain at most two clusters for each point and bring down the execution time of each iteration to $O(n^2)$ and the total execution time to $O(n^2/\eps^2)$.

\begin{lemma}\label{lem:twoclusters}
At any point in the execution of our algorithm,
given three copies of the same vertex $v \in A\cup B$ in $\mathcal{I}'$, at least two of them will share the same dual weight. 
\end{lemma}
\begin{proof}
Suppose there are three copies of $v \in B$ that have three distinct dual weights. Let $v_1$ and $v_2$ be two copies of $v$ such that 
\begin{equation}
    y(v_1) < y(v_2) -2\eps.
    \label{eq:match2}
\end{equation} From the fact that the dual weight of $v_2$ is larger than that of $v_1$, we conclude that $v_1$ is matched to some $a \in \mathbb{A}$. From~\eqref{eq:feas2}, we have \begin{equation}y(v_1)+ y(a) = \nc(a,b).\label{eq:match}
\end{equation} Now consider the edge $(v_2, a)$. From~\eqref{eq:feas1}, $y(v_2)+ y(a) \le \nc(a,b)+\eps = y(a)+y(v_1)+\eps$. The last equality follows from~\eqref{eq:match}. This implies $y(v_2) \le y(v_1)+\eps$ contradicting~\eqref{eq:match2}.
\end{proof}

\begin{theorem}
The push-relabel algorithm can compute an $\eps$-approximation to the optimal transport in $O(n^2/\eps^2)$ sequential time and $O(\log (n/\eps)/\eps^2)$ parallel time.
\end{theorem}

\section{Experiments} 
\label{sec:experiments}

\begin{figure}
\centering
\includegraphics[width=0.65\textwidth]{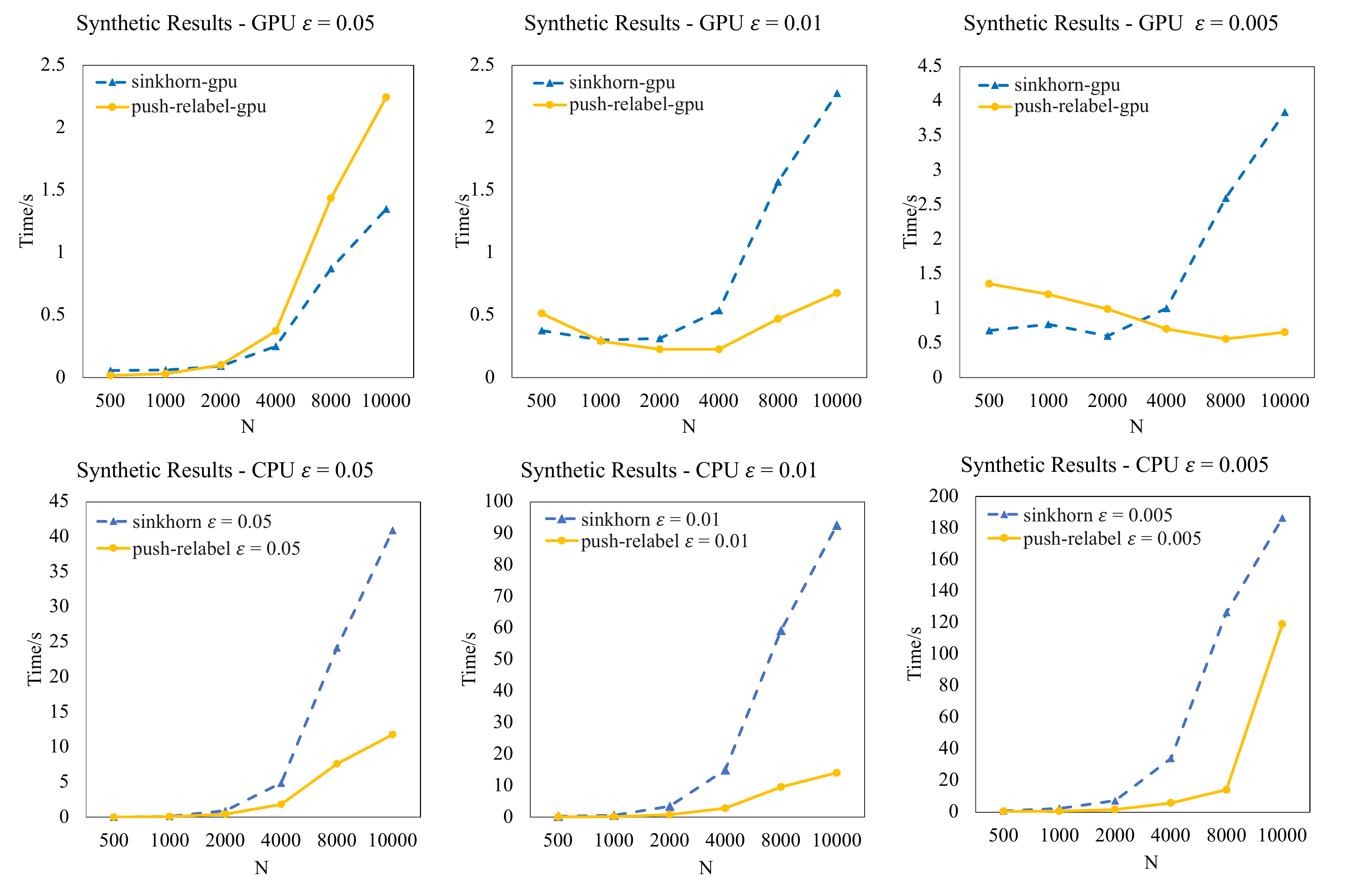}
\caption{A plot of running times for the synthetic inputs.}
\label{fig:experimental-results-synthetic}
\end{figure}
In this section, we present our experimental results. In our experiments, we compare an implementation of our $\eps$-approximate push-relabel based algorithm for solving the assignment problem with an implementation of the Sinkhorn algorithm~\cite{sinkhorn} for computing an $\eps$-approximate optimal transport plan. We implement two versions of our algorithm, both written in Python~\footnote{Our implementations and experimental setup can be found at https://github.com/kaiyiz/Push-Relabel-OT}. The first implementation is written with the NumPy library and uses only the CPU for computations, while the second implementation uses the CuPy library to execute some of the more expensive parts of the algorithm on a GPU. We compare these implementations of our algorithm to corresponding CPU-based and GPU-based implementations of the Sinkhorn algorithm. This Sinkhorn implementation is part of the Python Optimal Transport (POT) library~\cite{POT}. The POT library contains a CPU-based implementation of the Sinkhorn algorithm, as well as a GPU-based implementation.

Our experiments are run using an Intel Xeon E5-2680v4 2.4GHz chip and an Nvidia Tesla P100 GPU. We run experiments using both real and synthetic data. For the synthetic data, we generate each of the vertex sets $A$ and $B$ by sampling $n$ two-dimensional points uniformly at random from a unit square. For any pair of points $(a,b) \in A \times B$, the cost $c(a,b)$ is set as the Euclidean distance between the points $a$ and $b$. For each value of $\eps$ in $[0.1, 0.01, 0.005]$, and for each value of $n$ in $[500, 1000, 2000, 4000, 8000, 10000]$, we execute $30$ runs. For each run, we generate the cost matrix $c(\cdot, \cdot)$ and invoke both the Sinkhorn algorithm and our algorithm using several different values of $\eps$. Note that, for, in order to use the Sinkhorn algorithm to solve the assignment problem, we simply set the supply and demand of each vertex to $1 / n$. For each combination of $n$, $\eps$, and algorithm choice, we average the running times over all $30$ runs and record the results. The results can be seen in Figure~\ref{fig:experimental-results-synthetic}.

\begin{figure}
\centering
\includegraphics[width=0.55\textwidth]{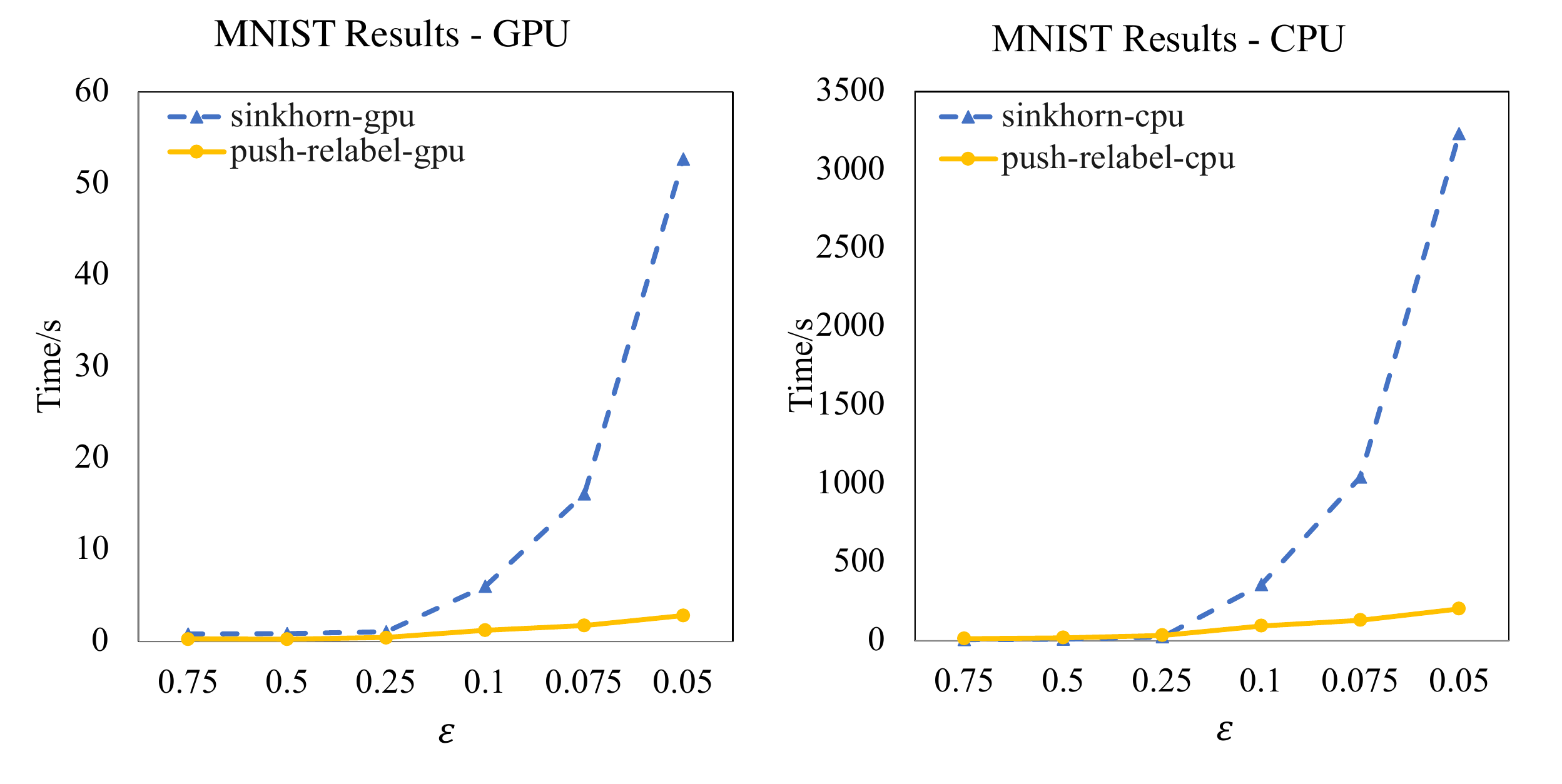}
\caption{A plot of running times for the MNIST image inputs.}
\label{fig:experimental-results-MNIST}
\end{figure}
Next, we run a similar experiment, using real-world data. We generate our inputs using the MNIST dataset of hand-written digit images~\cite{MNIST}. Each image consists of a $28 \times 28$ pixel gray-scale image. The sets $A$ and $B$ each consist of $n=10,000$ images from the MNIST dataset, selected at random. The cost $c(a,b)$ between two images $a \in A$ and $b \in B$ is computed as follows: Let $a(i,j)$ (resp. $b(i,j)$) be the value of the pixel in row $i$ and column $j$ of image $a$ (resp. $b$). First, the two images are normalized so that the sum of all pixel values is equal to $1$ for each image, i.e., $\sum_{i,j \in [1, 28]}a(i,j) = 1$ and $\sum_{i,j \in [1, 28]}b(i,j) = 1$. Then, the cost $c(a,b)$ is given by the $L_1$ distance between the resulting normalized images: $c(a,b) = \sum_{i,j \in [1, 28]}|a(i,j) - b(i,j)|$. Note that an upper bound on the largest cost is $2$. For each value of $\eps$ in $\{0.75, 0.5, 0.25, 0.1\}$, we execute both the CPU and GPU implementations of both our algorithm and the Sinkhorn algorithm. Once again, for each algorithm implementation and for each value of $\eps$, we average the results over $30$ runs. The results for these experiments can be found in Figure~\ref{fig:experimental-results-MNIST}.

Next, we summarize some observations from our experimental results. First, note that, for small values of $\eps$, the Sinkhorn algorithm's running time increases dramatically; this is especially noticeable for the MNIST results. This phenomena may be occurring because of issues resulting from numerical precision, due to Sinkhorn's regularization process, which uses an exponential scaling of a matrix. In contrast, our algorithm does not suffer from such issues with numerical precision. 

Overall, our results suggest that both the CPU-based and GPU-based implementations of our push-relabel algorithm seem to be competitive with the Sinkhorn algorithm. It is worth noting that this algorithm is very new in comparison to the Sinkhorn algorithm. As such, it is possible that a more careful implementation, especially for the GPU version,
could result in further improvements to actual running times.

\paragraph{Acknowledgment:} This research was partially supported by NSF CCF 1909171. The authors would like to thank Abhijeet Phatak and Chittaranjan Tripathy for helping us create a GPU implementation of our algorithm.

\bibliography{main}
\bibliographystyle{plain}
%%%%%%%%%%%%%%%%%%%%%%%%%%%%%%%%%%%%%%%%%%%%%%%%%%%%%%%%%%%%%%%%%%%%%%%%%%%%%%%
%%%%%%%%%%%%%%%%%%%%%%%%%%%%%%%%%%%%%%%%%%%%%%%%%%%%%%%%%%%%%%%%%%%%%%%%%%%%%%%
\end{document}